\documentclass[10pt,twocolumn,letterpaper]{article}

\usepackage[nocompress]{cite}
\usepackage{epsfig}
\usepackage{graphicx}
\usepackage{amsmath}
\usepackage{amssymb}
\usepackage{amsthm}
\usepackage{algorithm}
\usepackage{algorithmic}
\usepackage{xfrac}
\usepackage{url}

\usepackage[top=0.7in,bottom=0.7in,left=0.6in,right=0.7in,columnsep=0.35in]{geometry}
\newtheorem{assumption}{Assumption}

\newtheorem{theorem}{Theorem}
\newtheorem{claim}{Claim}

\newtheorem{lemma}{Lemma}
\DeclareMathOperator*{\argmax}{arg\,max}
\newcommand{\R}{\mathbb{R}}
\newcommand{\iverson}[1]{1\left({#1}\right)}
\DeclareMathOperator*{\E}{\mathbb{E}}
\let\P\undefined
\DeclareMathOperator*{\P}{\mathbb{P}}
\newcommand{\dotprod}[2]{#1 \cdot #2}
\newcommand{\norm}[1]{{\left\|#1\right\|}}
\renewcommand{\exp}[1]{e^{#1}}
\newcommand{\lexp}[1]{{\rm exp}\left({#1}\right)}
\newcommand{\ceil}[1]{\left\lceil #1 \right\rceil}
\newcommand{\floor}[1]{\left\lfloor #1 \right\rfloor}
\newcommand{\eps}{\varepsilon}
\renewcommand{\O}{\mathcal{O}}
\newcommand{\X}{\mathcal{X}}
\newcommand{\Y}{\mathcal{Y}}
\newcommand{\Yp}{\Y'}
\newcommand{\W}{\mathcal{W}}
\newcommand{\PS}{\mathcal{P}}
\newcommand{\yp}{y'}
\newcommand{\yh}{\hat{y}}
\renewcommand{\wp}{w'}
\newcommand{\m}{n'}
\newcommand{\T}{\mathbb{T}}
\newcommand{\Rademacher}{\mathfrak{R}}
\newcommand{\G}{\mathfrak{G}}
\renewcommand{\H}{\mathfrak{H}}
\newcommand{\wh}{\widehat{w}}
\newcommand{\Rp}{R'}
\newcommand{\rp}{r'}
\renewcommand{\sp}{s'}
\newcommand{\A}{\mathcal{A}}
\newcommand{\Sp}{S'}
\newcommand{\pip}{\pi'}
\newcommand{\hns}{\hspace{-0.025in}}
\newcommand{\s}{\mathfrak{s}}
\newcommand{\ft}{\widetilde{f}}
\newcommand{\npcite}[1]{\cite{#1}}
\date{\vspace{-0.3in}}

\title{Structured Prediction: From Gaussian Perturbations\\
to Linear-Time Principled Algorithms}

\author{
Jean Honorio \\
CS, Purdue \\
West Lafayette, IN 47907, USA \\
\texttt{jhonorio@purdue.edu}
\and
Tommi Jaakkola \\
CSAIL, MIT \\
Cambridge, MA 02139, USA \\
\texttt{tommi@csail.mit.edu}
}

\begin{document}

\maketitle

\begin{abstract}

Margin-based structured prediction commonly uses a maximum loss over all possible structured outputs \cite{Altun03,Collins04b,Taskar03}.
In natural language processing, recent work \cite{Zhang14,Zhang15} has proposed the use of the maximum loss over random structured outputs sampled independently from some proposal distribution.
This method is linear-time in the number of random structured outputs and trivially parallelizable.
We study this family of loss functions in the PAC-Bayes framework under Gaussian perturbations \cite{McAllester07}.
Under some technical conditions and up to statistical accuracy, we show that this family of loss functions produces a tighter upper bound of the Gibbs decoder distortion than commonly used methods.
Thus, using the maximum loss over random structured outputs is a principled way of learning the parameter of structured prediction models.
Besides explaining the experimental success of \cite{Zhang14,Zhang15}, our theoretical results show that more general techniques are possible.

\end{abstract}

\section{Introduction}

Structured prediction has been shown to be useful in many diverse domains.
Application areas include natural language processing (e.g., named entity recognition, part-of-speech tagging, dependency parsing), computer vision (e.g., image segmentation, multiple object tracking), speech (e.g., text-to-speech mapping) and computational biology (e.g., protein structure prediction).

In dependency parsing, for instance, the observed input is a sentence and the desired structured output is a parse tree for the given sentence.

In general, structured prediction can be viewed as a kind of decoding.
A \emph{decoder} is a machine for predicting the structured output $y$ given the observed input $x$.
Such a decoder, depends on a parameter $w$.
Given a fixed $w$, the task performed by the decoder is called \emph{inference}.
In this paper, we focus on the problem of learning the parameter $w$.
Next, we introduce the problem and our main contributions.

We assume a distribution $D$ on pairs ${(x,y)}$ where ${x \in \X}$ is the observed input and ${y \in \Y}$ is the latent structured output, i.e., ${(x,y) \sim D}$.
We also assume that we have a training set $S$ of $n$ i.i.d. samples drawn from the distribution $D$, i.e., ${S \sim D^n}$, and thus ${|S|=n}$.

We let ${\Y(x) \neq \emptyset}$ denote the countable set of feasible \emph{decodings} of $x$.
In general, ${|\Y(x)|}$ is exponential with respect to the input size.

We assume a fixed mapping $\phi$ from pairs to feature vectors, i.e., for any pair ${(x,y)}$ we have the feature vector ${\phi(x,y) \in \R^k \setminus \{0\}}$.
For a parameter ${w \in \W \subseteq \R^k \setminus \{0\}}$, we consider linear decoders of the form:
\begin{align} \label{eq:inferenceall}
f_w(x) \equiv \argmax_{y \in \Y(x)}{\dotprod{\phi(x,y)}{w}}
\end{align}
In practice, very few cases of the above general \emph{inference} problem are tractable, while most are NP-hard and also hard to approximate within a fixed factor.
(We defer the details in theory of computation to Section \ref{sec:discussion}.)

We also introduce the \emph{distortion} function ${d : \Y \times \Y \to [0,1]}$.
The value ${d(y,\yp)}$ measures the amount of difference between two structured outputs $y$ and $\yp$.
Disregarding the computational and statistical aspects, the ultimate goal is to set the parameter $w$ in order to minimize the decoder distortion.
That is:
\begin{align} \label{eq:nonrobust}
\min_{w \in \W}{ \E_{(x,y) \sim D}\left[d(y,f_w(x))\right] }
\end{align}
Computationally speaking, the above procedure is inefficient since ${d(y,f_w(x))}$ is a discontinuous function with respect to $w$ and thus, it is in general an exponential-time optimization problem.
Statistically speaking, the problem in eq.\eqref{eq:nonrobust} requires access to the data distribution $D$ and thus, in general it would require an infinite amount of data.
In practice, we only have access to a small amount of training data.

Additionally, eq.\eqref{eq:nonrobust} would potentially favor parameters $w$ with low distortion, but that could be in a neighborhood of parameters with high distortion.
In order to avoid this issue, we could optimize a more ``robust'' objective under Gaussian perturbations.
More formally, let ${\alpha>0}$ and let ${Q(w)}$ be a unit-variance Gaussian distribution centered at ${w \alpha}$ of parameters ${\wp \in \W}$.
The Gibbs decoder distortion of the perturbation distribution ${Q(w)}$ and data distribution $D$, is defined as:
\begin{align} \label{eq:gibbsdd}
L(Q(w),D) = \E_{(x,y) \sim D}\left[\E_{\wp \sim Q(w)}[d(y,f_{\wp}(x))]\right]
\end{align}
The minimization of the Gibbs decoder distortion can be expressed as:
\begin{align*}
\min_{w \in \W}{ L(Q(w),D) }
\end{align*}
The focus of our analysis will be to propose upper bounds of the Gibbs decoder distortion, with good computational and statistical properties.
That is, we will propose upper bounds that can be computed in polynomial-time, and that require a small amount of training data.

For our analysis, we follow the same set of assumptions as in \cite{McAllester07}.
We define the margin ${m(x,y,\yp,w)}$ as the amount by which $y$ is preferable to $\yp$ under the parameter $w$.
More formally:
\begin{align*}
m(x,y,\yp,w) \equiv \dotprod{\phi(x,y)}{w} - \dotprod{\phi(x,\yp)}{w}
\end{align*}
Let ${c(p,x,y)}$ be a nonnegative integer that gives the number of times that the part ${p \in \PS}$ appears in the pair ${(x,y)}$.
For a part ${p \in \PS}$, we define the feature $p$ as follows:
\begin{align*}
\phi_p(x,y) \equiv c(p,x,y)
\end{align*}
We let ${\PS(x) \neq \emptyset}$ denote the set of ${p \in \PS}$ such that there exists ${y \in \Y(x)}$ with ${c(p,x,y)>0}$.
We define the Hamming distance $H$ as follows:
\begin{align*}
H(x,y,\yp) \equiv \sum_{p \in \PS(x)}{|c(p,x,y) - c(p,x,\yp)|}
\end{align*}
The commonly applied margin-based approach to learning $w$ uses the maximum loss over all possible structured outputs \cite{Altun03,Collins04b,Taskar03}.
That is:\footnote{
\label{foo:hinge}For computational convenience, the \emph{convex} hinge loss ${\max{(0,1+z)}}$ is used in practice instead of the \emph{discontinuous} 0/1 loss ${\iverson{z \geq 0}}$.}
\begin{align} \label{eq:trainall}
 & \min_{w \in \W}{ \frac{1}{n} \sum_{(x,y) \in S}{ \max_{\yh \in \Y(x)}{d(y,\yh) {\rm\ } \iverson{\begin{array}{@{}l@{}}
  H(x,y,\yh) \\
  - m(x,y,\yh,w) \geq 0
  \end{array}}} } } \nonumber \\
 & \hspace{0.3in} + \lambda \norm{w}_2^2
\end{align}
In Section \ref{sec:pacbayesall}, we reproduce the results in \cite{McAllester07} and show that the above objective is related to an upper bound of the Gibbs decoder distortion in eq.\eqref{eq:gibbsdd}.
Note that evaluating the objective function in eq.\eqref{eq:trainall} is as hard as the inference problem in eq.\eqref{eq:inferenceall}, since both perform maximization over the set ${\Y(x)}$.

Our main contributions are presented in Sections \ref{sec:pacbayesrandom} and \ref{sec:examples}.
Inspired by recent work in natural language processing \cite{Zhang14,Zhang15}, we show a tighter upper bound of the Gibbs decoder distortion in eq.\eqref{eq:gibbsdd}, which is related to the following objective:\textsuperscript{\ref{foo:hinge}}
\begin{align} \label{eq:trainrandom}
 & \min_{w \in \W}{ \frac{1}{n} \sum_{(x,y) \in S}{ \max_{\yh \in T(w,x)}{d(y,\yh) {\rm\ } \iverson{\begin{array}{@{}l@{}}
  H(x,y,\yh) \\
  - m(x,y,\yh,w) \geq 0
  \end{array}}} } } \nonumber \\
 & \hspace{0.3in} + \lambda \norm{w}_2^2
\end{align}
\noindent where ${T(w,x)}$ is a set of random structured outputs sampled i.i.d. from some proposal distribution with support on ${\Y(x)}$.
Note that evaluating the objective function in eq.\eqref{eq:trainrandom} is linear-time in the number of random structured outputs in ${T(w,x)}$.

\section{From PAC-Bayes to the Maximum Loss Over All Possible Structured Outputs} \label{sec:pacbayesall}

In this section, we show the relationship between PAC-Bayes bounds and the commonly used maximum loss over all possible structured outputs.

As reported in \cite{McAllester07}, by using the PAC-Bayes framework under Gaussian perturbations, we show that the commonly used maximum loss over all possible structured outputs is an upper bound of the Gibbs decoder distortion up to statistical accuracy (${\O(\sqrt{\sfrac{\log{n}}{n}})}$ for $n$ training samples).

\begin{theorem}[\npcite{McAllester07}] \label{thm:pacbayesall}
Assume that there exists a finite integer value $\ell$ such that ${|\cup_{(x,y) \in S}{\PS(x)}| \leq \ell}$.
Fix ${\delta \in (0,1)}$.
With probability at least ${1-\delta/2}$ over the choice of $n$ training samples, simultaneously for all parameters ${w \in \W}$ and unit-variance Gaussian perturbation distributions ${Q(w)}$ centered at ${w \sqrt{2\log{(2 n \ell/\norm{w}_2^2)}}}$, we have:
\begin{align*}
L(Q(w),D) \\
 & \hspace{-0.4in} \leq \frac{1}{n} \sum_{(x,y) \in S}{ \max_{\yh \in \Y(x)}{d(y,\yh) {\rm\ } \iverson{\begin{array}{@{}l@{}}
 H(x,y,\yh) \\
 - m(x,y,\yh,w) \geq 0
 \end{array}}} } \\
 & \hspace{-0.4in} + \frac{\norm{w}_2^2}{n} + \sqrt{\frac{\norm{w}_2^2 \log{(2 n \ell/\norm{w}_2^2)} + \log{(2n/\delta)}}{2(n-1)}}
\end{align*}
\end{theorem}
(See Appendix \ref{sec:detailedproofs} for detailed proofs.)

The proof of the above is based on the PAC-Bayes theorem and well-known Gaussian concentration inequalities.
As it is customary in generalization results, a \emph{deterministic} expectation with respect to the data distribution $D$ is upper-bounded by a \emph{stochastic} quantity with respect to the training set $S$.
This takes into account the statistical aspects of the problem.

Note that the upper bound uses maximization with respect to ${\Y(x)}$ and that in general, ${|\Y(x)|}$ is exponential with respect to the input size.
Thus, the computational aspects of the problem have not been fully addressed yet.
In the next section, we solve this issue by introducing randomness.

\section{From PAC-Bayes to the Maximum Loss Over Random Structured Outputs} \label{sec:pacbayesrandom}

In this section, we analyze the relationship between PAC-Bayes bounds and the maximum loss over random structured outputs sampled independently from some proposal distribution.

First, we will focus on the computational aspects.
Instead of using maximization with respect to ${\Y(x)}$, we will perform maximization with respect to a set ${T(w,x)}$ of random structured outputs sampled i.i.d. from some proposal distribution ${R(w,x)}$ with support on ${\Y(x)}$.
In order for this approach to be computationally appealing, ${|T(w,x)|}$ should be polynomial, even when ${|\Y(x)|}$ is exponential with respect to the input size.

Assumptions \ref{asm:maxdistortion} and \ref{asm:lownorm} will allow us to attain ${|T(w,x)| = \O\left(\max{\left(\frac{1}{\log{(1/\beta)}}, \norm{w}_2^2\right)}\right)}$.
The constant ${\beta \in [0,1)}$ is properly introduced on Assumption \ref{asm:maxdistortion}.
It can be easily observed that $\beta$ plays an important role in the number of random structured outputs that we need to draw from the proposal distribution ${R(w,x)}$.
Next, we present our first assumption.

\begin{assumption}[Maximal distortion] \label{asm:maxdistortion}
The proposal distribution ${R(w,x)}$ fulfills the following condition.
There exists a value ${\beta \in [0,1)}$ such that for all ${(x,y) \in S}$ and ${w \in \W}$:
\begin{align*}
\P_{\yp \sim R(w,x)}[d(y,\yp)=1] \geq 1-\beta
\end{align*}
\end{assumption}

In Section \ref{sec:examples} we show examples that fulfill the above assumption, which include a \emph{binary} distortion function for \emph{any} type of structured output, as well as a distortion function that returns the number of different edges/elements for directed spanning trees, directed acyclic graphs and cardinality-constrained sets.

Next, we present our second assumption that allows obtaining ${|T(w,x)| = \O\left(\max{\left(\frac{1}{\log{(1/\beta)}}, \norm{w}_2^2\right)}\right)}$.
While Assumption \ref{asm:maxdistortion} contributes with the term ${\frac{1}{\log{(1/\beta)}}}$ in ${|T(w,x)|}$, the following assumption contributes with the term ${\norm{w}_2^2}$ in ${|T(w,x)|}$.

\begin{assumption}[Low norm] \label{asm:lownorm}
For any vector ${z \in \R^k}$, define:
\begin{align*}
\mu(z) = \begin{cases}
z/\norm{z}_1 & \text{ if } z \neq 0 \\
0 & \text{ if } z = 0
\end{cases}
\end{align*}
The proposal distribution ${R(w,x)}$ fulfills the following condition for all ${(x,y) \in S}$ and ${w \in \W}$:\footnote{The second inequality follows from an implicit assumption made in Theorem \ref{thm:pacbayesall}, i.e., ${\norm{w}_2^2/n \leq 1}$.
Note that if ${\norm{w}_2^2/n > 1}$ then Theorem \ref{thm:pacbayesall} provides an upper bound greater than $1$, which is meaningless since the distortion function $d$ is at most $1$.}
\begin{align*}
\norm{\E_{\yp \sim R(w,x)}\left[ \mu(\phi(x,y) - \phi(x,\yp)) \right]}_2 \leq \frac{1}{2 \sqrt{n}} \leq \frac{1}{2 \norm{w}_2}
\end{align*}
\end{assumption}

It is natural to ask whether there are instances that fulfill the above assumption.
In Section \ref{sec:examples} we provide two extreme cases: one example of a \emph{sparse} mapping and a uniform proposal, and one example of a \emph{dense} mapping and an \emph{arbitrary} proposal distribution.

We will now focus on the statistical aspects.
Note that randomness does not only stem from data, but also from sampling structured outputs.
That is, in Theorem \ref{thm:pacbayesall}, randomness only stems from the training set $S$.
We now need to produce generalization results that hold for all the sets ${T(w,x)}$ of random structured outputs.
In addition, the uniform convergence of Theorem \ref{thm:pacbayesall} holds for all parameters $w$.
We now need to produce a generalization result that also holds for all possible  proposal distributions ${R(w,x)}$.
Therefore, we need a method for upper-bounding the number of possible proposal distributions ${R(w,x)}$.
Assumption \ref{asm:linearordering} will allow us to upper-bound this number.

\begin{assumption}[Linearly inducible ordering] \label{asm:linearordering}
The proposal distribution ${R(w,x)}$ depends solely on the linear ordering induced by the parameter ${w \in \W}$ and the mapping ${\phi(x,\cdot)}$.
More formally, let ${r(x) \equiv |\Y(x)|}$ and thus ${\Y(x) \equiv \{y_1 \dots y_{r(x)}\}}$.
Let ${w,\wp \in \W}$ be any two arbitrary parameters.
Let ${\pi(x) = (\pi_1 \dots \pi_{r(x)})}$ be a permutation of ${\{1 \dots r(x)\}}$ such that ${\dotprod{\phi(x,y_{\pi_1})}{w} < \dots < \dotprod{\phi(x,y_{\pi_{r(x)}})}{w}}$.
Let ${\pip(x) = (\pip_1 \dots \pip_{r(x)})}$ be a permutation of ${\{1 \dots r(x)\}}$ such that ${\dotprod{\phi(x,y_{\pip_1})}{\wp} < \dots < \dotprod{\phi(x,y_{\pip_{r(x)}})}{\wp}}$.
For all ${w,\wp \in \W}$ and ${x \in \X}$, if ${\pi(x)=\pip(x)}$ then ${KL(R(w,x) \| R(\wp,x))=0}$.
In this case, we say that the proposal distribution fulfills ${R(\pi(x),x) \equiv R(w,x)}$.
\end{assumption}

Assumption \ref{asm:linearordering} states that two proposal distributions ${R(w,x)}$ and ${R(\wp,x)}$ are the same provided that for the same permutation ${\pi(x)}$ we have ${\dotprod{\phi(x,y_{\pi_1})}{w} < \dots < \dotprod{\phi(x,y_{\pi_{r(x)}})}{w}}$ and ${\dotprod{\phi(x,y_{\pi_1})}{\wp} < \dots < \dotprod{\phi(x,y_{\pi_{r(x)}})}{\wp}}$.
Geometrically speaking, for a fixed $x$ we first project the feature vectors ${\phi(x,y)}$ of all the structured outputs ${y \in \Y(x)}$ onto the lines $w$ and $\wp$.
Let ${\pi(x)}$ and ${\pip(x)}$ be the resulting ordering of the structured outputs after projecting them onto $w$ and $\wp$ respectively.
Two proposal distributions ${R(w,x)}$ and ${R(\wp,x)}$ are the same provided that ${\pi(x) = \pip(x)}$.
That is, the specific values of ${\dotprod{\phi(x,y)}{w}}$ and ${\dotprod{\phi(x,y)}{\wp}}$ are irrelevant, and only their ordering matters.

In Section \ref{sec:examples} we show examples that fulfill the above assumption, which include the algorithm proposed in \cite{Zhang14,Zhang15} for directed spanning trees, and our proposed generalization to any type of data structure with computationally efficient local changes.

In what follows, by using the PAC-Bayes framework under Gaussian perturbations, we show that the maximum loss over random structured outputs sampled independently from some proposal distribution provides an upper bound of the Gibbs decoder distortion up to statistical accuracy (${\O(\sfrac{\log^{3/2}{n}}{\sqrt{n}})}$ for $n$ training samples).

\begin{theorem} \label{thm:pacbayesrandom}
Assume that there exist finite integer values $\ell$ and $r$ such that ${|\cup_{(x,y) \in S}{\PS(x)}| \leq \ell}$ and ${|\Y(x)| \leq r}$ for all ${(x,y) \in S}$.
Assume that the proposal distribution ${R(w,x)}$ with support on ${\Y(x)}$ fulfills Assumption \ref{asm:maxdistortion} with value $\beta$, as well as Assumptions \ref{asm:lownorm} and \ref{asm:linearordering}.
Fix ${\delta \in (0,1)}$ and an integer $\s$ such that ${3 \leq \s \leq \frac{9}{20} \sqrt{\ell+1}}$.
With probability at least ${1-\delta}$ over the choice of both $n$ training samples and $n$ sets of random structured outputs, simultaneously for all parameters ${w \in \W}$ with ${\norm{w}_0 \leq \s}$, unit-variance Gaussian perturbation distributions ${Q(w)}$ centered at ${w \sqrt{2\log{(2 n \ell/\norm{w}_2^2)}}}$, and for sets of random structured outputs ${T(w,x)}$ sampled i.i.d. from the proposal distribution ${R(w,x)}$ for each training sample ${(x,y) \in S}$, such that ${|T(w,x)| = \ceil{\frac{1}{2} \max{\left(\frac{1}{\log{(1/\beta)}}, 32 \norm{w}_2^2\right)} \log{n}}}$, we have:
\begin{align*}
L(Q(w),D) \\
 & \hspace{-0.725in} \leq \frac{1}{n} \sum_{(x,y) \in S}{ \max_{\yh \in T(w,x)}{d(y,\yh) {\rm\ } \iverson{\begin{array}{@{}l@{}}
  H(x,y,\yh) \\
  - m(x,y,\yh,w) \geq 0
  \end{array}}} } \\
 & \hspace{-0.725in} + \frac{\norm{w}_2^2}{n} + \sqrt{\frac{\norm{w}_2^2 \log{(2 n \ell/\norm{w}_2^2)} + \log{(2n/\delta)}}{2(n-1)}} + \sqrt{\frac{1}{n}} \\
 & \hspace{-0.725in} + {\textstyle \max{\hns\left(\frac{1}{\log{(1/\beta)}}, 32 \norm{w}_2^2\right)} } \sqrt{\frac{\s \log{(\ell\hns+\hns 1)} \log^3\hns{(n\hns+\hns 1)}}{n}} \\
 & \hspace{-0.725in} + 3 \sqrt{\frac{\s (\log{\ell}+2 \log{(nr)})+\log{(4/\delta)}}{n}}
\end{align*}
\end{theorem}
(See Appendix \ref{sec:detailedproofs} for detailed proofs.)

The proof of the above is based on Theorem \ref{thm:pacbayesall} as a starting point.
In order to account for the computational aspect of requiring sets ${T(w,x)}$ of polynomial size, we use Assumptions \ref{asm:maxdistortion} and \ref{asm:lownorm} for bounding a \emph{deterministic} expectation.
In order to account for the statistical aspects, we use Assumption \ref{asm:linearordering} and Rademacher complexity arguments for bounding a \emph{stochastic} quantity for all sets ${T(w,x)}$ of random structured outputs and all possible proposal distributions ${R(w,x)}$.
The assumption of sparsity (i.e., ${\norm{w}_0 \leq \s}$) is pivotal for obtaining terms of order ${\O(\sqrt{\sfrac{\s \log{\ell}}{n}}))}$.
Without sparsity, the terms would be of order ${\O(\sqrt{\sfrac{\ell}{n}})}$ which is not suited for high-dimensional settings.

\paragraph{Inference on Test Data.}

Note that the upper bound in Theorem \ref{thm:pacbayesrandom} holds simultaneously for all parameters ${w \in \W}$.
Therefore, our result implies that after learning the optimal parameter ${\wh \in \W}$ in eq.\eqref{eq:trainrandom} from \emph{training} data, we can bound the decoder distortion when performing \emph{exact} inference on \emph{test} data.
More formally, Theorem \ref{thm:pacbayesrandom} can be additionally invoked for a \emph{test} set $\Sp$, also with probability at least ${1-\delta}$.
Thus, under the same setting as of Theorem \ref{thm:pacbayesrandom}, the Gibbs decoder distortion is upper-bounded with probability at least ${1-2\delta}$ over the choice of $S$ and $\Sp$.
In this paper, we focus on learning the parameter of structured prediction models.
We leave the analysis of \emph{approximate} inference on test data for future work.

\section{Examples} \label{sec:examples}

In this section, we provide several examples that fulfill the three main assumptions of our theoretical result.

\subsection{Examples for the Maximal Distortion Assumption}

In what follows, we present some examples that fulfill our Assumption \ref{asm:maxdistortion}.
For a \emph{binary} distortion function, we show that \emph{any} type of structured output fulfills the above assumption.
For a distortion function that returns the number of different edges/elements, we show that directed spanning trees, directed acyclic graphs and cardinality-constrained sets, fulfill the assumption as well.

For simplicity of analysis, most proofs in this part will assume a uniform proposal distribution ${R(w,x) = R(x)}$ with support on ${\Y(x)}$.
In the following claim, we argue that we can perform a change of measure between different proposal distributions.
Thus, allowing us to focus on uniform proposals afterwards.

\begin{claim}[Change of measure] \label{clm:changeofmeasure}
Let ${R(w,x)}$ and ${\Rp(w,x)}$ two proposal distributions, both with support on ${\Y(x)}$.
Assume that the proposal distribution ${R(w,x)}$ fulfills Assumption \ref{asm:maxdistortion} with value ${\beta_1}$.
Let ${r_{w,x}(\cdot)}$ and ${\rp_{w,x}(\cdot)}$ be the probability mass functions of ${R(w,x)}$ and ${\Rp(w,x)}$ respectively.
Assume that the total variation distance between ${R(w,x)}$ and ${\Rp(w,x)}$ is bounded as follows for all ${(x,y) \in S}$ and ${w \in \W}$:
\begin{align*}
TV(R(w,x) \| \Rp(w,x)) & \equiv \frac{1}{2} \sum_{y \in \Y(x)}{|r_{w,x}(y) - \rp_{w,x}(y)|} \\
 & \leq \beta_2
\end{align*}
The proposal distribution ${\Rp(w,x)}$ fulfills Assumption \ref{asm:maxdistortion} with ${\beta = \beta_1 + \beta_2}$ provided that ${\beta_1 + \beta_2 \in [0,1)}$.
\end{claim}

Next, we provide a result for \emph{any} type of structured output, but for a \emph{binary} distortion function.

\begin{claim}[Any type of structured output] \label{clm:anystruct}
Let ${\Y(x)}$ be an arbitrary countable set of feasible decodings of $x$, such that ${|\Y(x)| \geq 2}$ for all ${(x,y) \in S}$.
Let ${d(y,\yp) = \iverson{y \neq \yp}}$.
The uniform proposal distribution ${R(w,x) = R(x)}$ with support on ${\Y(x)}$ fulfills Assumption \ref{asm:maxdistortion} with ${\beta = 1/2}$.
\end{claim}

The following claim pertains to directed spanning trees and for a distortion function that returns the number of different edges.

\begin{claim}[Directed spanning trees] \label{clm:trees}
Let ${\Y(x)}$ be the set of directed spanning trees of $v$ nodes.
Let ${A(y)}$ be the adjacency matrix of ${y \in \Y(x)}$.
Let ${d(y,\yp) = \frac{1}{2 (v-1)} \sum_{ij}{|A(y)_{ij} - A(\yp)_{ij}|}}$.
The uniform proposal distribution ${R(w,x) = R(x)}$ with support on ${\Y(x)}$ fulfills Assumption \ref{asm:maxdistortion} with ${\beta = \frac{v-2}{v-1}}$.
\end{claim}

The next result is for directed acyclic graphs and for a distortion function that returns the number of different edges.

\begin{claim}[Directed acyclic graphs] \label{clm:dags}
Let ${\Y(x)}$ be the set of directed acyclic graphs of $v$ nodes and $b$ parents per node, such that ${2 \leq b \leq v-2}$.
Let ${A(y)}$ be the adjacency matrix of ${y \in \Y(x)}$.
Let ${d(y,\yp) = \frac{1}{b(2v-b-1)} \sum_{ij}{|A(y)_{ij} - A(\yp)_{ij}|}}$.
The uniform proposal distribution ${R(w,x) = R(x)}$ with support on ${\Y(x)}$ fulfills Assumption \ref{asm:maxdistortion} with ${\beta = \frac{b^2+2b+2}{b^2+3b+2}}$.
\end{claim}

The final example is for cardinality-constrained sets and for a distortion function that returns the number of different elements.

\begin{claim}[Cardinality-constrained sets] \label{clm:cardsets}
Let ${\Y(x)}$ be the set of sets of $b$ elements chosen from $v$ possible elements, such that ${b \leq v/2}$.
Let ${d(y,\yp) = \frac{1}{2 b} (|y-\yp| + |\yp-y|)}$.
The uniform proposal distribution ${R(w,x) = R(x)}$ with support on ${\Y(x)}$ fulfills Assumption \ref{asm:maxdistortion} with ${\beta = 1/2}$.
\end{claim}

\subsection{Examples for the Low Norm Assumption}

Next, we present some examples that fulfill our Assumption \ref{asm:lownorm}.
We provide two extreme cases: one example for \emph{sparse} mappings, and one example for \emph{dense} mappings.

Next, we provide a result for a particular instance of a sparse mapping and a uniform proposal distribution.

\begin{claim}[Sparse mapping] \label{clm:sparsedata}
Let ${b>0}$ be an arbitrary integer value.
For all ${(x,y) \in S}$, let ${\Y(x) = \cup_{p \in \PS(x)}{\Y_p(x)}}$, where the partition ${\Y_p(x)}$ is defined as follows:
\begin{align*}
(\forall p \in \PS(x)) {\rm\ } \Y_p(x) \equiv \{ & \yp \; \mid \; |\phi_p(x,y) - \phi_p(x,\yp)| = b \; \wedge \\
 & (\forall q \neq p) {\rm\ } \phi_q(x,y) = \phi_q(x,\yp) \}
\end{align*}
If ${n \leq |\PS(x)|/4}$ for all ${(x,y) \in S}$, then the uniform proposal distribution ${R(w,x) = R(x)}$ with support on ${\Y(x)}$ fulfills Assumption \ref{asm:lownorm}.
\end{claim}

The following claim pertains to a particular instance of a dense mapping and an \emph{arbitrary} proposal distribution.

\begin{claim}[Dense mapping] \label{clm:densedata}
Let ${b>0}$ be an arbitrary integer value.
Let ${|\phi_p(x,y) - \phi_p(x,\yp)| = b}$ for all ${(x,y) \in S}$, ${\yp \in \Y(x)}$ and ${p \in \PS(x)}$.
If ${n \leq |\PS(x)|/4}$ for all ${(x,y) \in S}$, then any arbitrary proposal distribution ${R(w,x)}$ fulfills Assumption \ref{asm:lownorm}.
\end{claim}

\subsection{Examples for the Linearly Inducible Ordering Assumption}

In what follows, we present some examples that fulfill our Assumption \ref{asm:linearordering}.
We show that the algorithm proposed in \cite{Zhang14,Zhang15} for directed spanning trees, fulfills the above assumption.
We also generalize the algorithm in \cite{Zhang14,Zhang15} to any type of data structure with computationally efficient local changes, and show that this generalization fulfills the assumption as well.

Next, we present the algorithm proposed in \cite{Zhang14,Zhang15} for dependency parsing in natural language processing.
Here, $x$ is a sentence of $v$ words and ${\Y(x)}$ is the set of directed spanning trees of $v$ nodes.

\begin{algorithm}[H]
\caption{Procedure for sampling a directed spanning tree ${\yp \in \Y(x)}$ from a greedy local proposal distribution ${R(w,x)}$}
\label{alg:zhang}
\begin{small}
\begin{algorithmic}
\STATE \textbf{Input:} parameter $w \in \W$, sentence $x \in \X$
\STATE Draw uniformly at random a directed spanning tree ${\yh \in \Y(x)}$
\REPEAT
  \STATE ${s \leftarrow }$ post-order traversal of $\yh$
  \FOR{ each node $t$ in the list $s$}
    \FOR{ each node $u$ before $t$ in the list $s$}
    \STATE ${y \leftarrow }$ change the parent of node $t$ to $u$ in $\yh$
      \IF{ ${\dotprod{\phi(x,y)}{w} > \dotprod{\phi(x,\yh)}{w}}$}
        \STATE ${\yh \leftarrow y}$
      \ENDIF
     \ENDFOR
  \ENDFOR
\UNTIL no refinement in last iteration
\STATE \textbf{Output:} directed spanning tree ${\yp \leftarrow \yh}$
\end{algorithmic}
\end{small}
\end{algorithm}

The above algorithm has the following property:

\begin{claim}[Sampling for directed spanning trees] \label{clm:zhang}
Algorithm \ref{alg:zhang} fulfills Assumption \ref{asm:linearordering}.
\end{claim}

Note that Algorithm \ref{alg:zhang} proposed in \cite{Zhang14,Zhang15} uses the fact that we can perform local changes to a directed spanning tree in a computationally efficient manner.
That is, changing parents of nodes in a post-order traversal will produce directed spanning trees.
We can extend the above algorithm to any type of data structure where we can perform computationally efficient local changes.
For instance, we can easily extend the method for directed acyclic graphs (traversed in post-order as well) and for sets with up to some prespecified number of elements.

Next, we generalize Algorithm \ref{alg:zhang} to any type of structured output.

\begin{algorithm}[H]
\caption{Procedure for sampling a structured output ${\yp \in \Y(x)}$ from a greedy local proposal distribution ${R(w,x)}$}
\label{alg:greedylocal}
\begin{small}
\begin{algorithmic}
\STATE \textbf{Input:} parameter $w \in \W$, observed input $x \in \X$
\STATE Draw uniformly at random a structured output ${\yh \in \Y(x)}$
\REPEAT
  \STATE Make a local change to $\yh$ in order to increase ${\dotprod{\phi(x,\yh)}{w}}$
\UNTIL no refinement in last iteration
\STATE \textbf{Output:} structured output ${\yp \leftarrow \yh}$
\end{algorithmic}
\end{small}
\end{algorithm}

The above algorithm has the following property:

\begin{claim}[Sampling for any type of structured output] \label{clm:greedylocal}
Algorithm \ref{alg:greedylocal} fulfills Assumption \ref{asm:linearordering}.
\end{claim}

\section{Experimental Results} \label{sec:experiments}

\begin{table*}
\caption{Average over $30$ repetitions, and standard error at 95\% confidence level of several methods and measurements.
For the maximum loss over all possible structured outputs (All) we used eq.\eqref{eq:trainall} for training, and eq.\eqref{eq:inferenceall} for inference on a test set.
For the maximum loss over random structured outputs (Random and Random/All) we used eq.\eqref{eq:trainrandom} for training.
For inference, Random used eq.\eqref{eq:inferencerandom} while Random/All used eq.\eqref{eq:inferenceall}.
Random outperforms All in the different study cases (directed spanning trees, directed acyclic graphs and cardinality-constrained sets).
The difference between Random and Random/All is not statistically significant.}
\label{tab:results}
\begin{small}
\begin{center}
\begin{tabular}{@{}l@{\hspace{0.125in}}l@{\hspace{0.125in}}c@{\hspace{0.125in}}c@{\hspace{0.125in}}c@{\hspace{0.125in}}c@{\hspace{0.125in}}c@{\hspace{0.125in}}c@{}}
\hline
\textbf{Problem} & \textbf{Method} & \textbf{Training} & \textbf{Training} & \textbf{Test} & \textbf{Test} & \textbf{Distance to} & \textbf{Angle with} \\
 & & \textbf{runtime} & \textbf{distortion} & \textbf{runtime} & \textbf{distortion} & \textbf{ground truth} & \textbf{ground truth} \\
\hline
Directed & All & 1000 & 52\% $\pm$ 1.1\% & 12.4 $\pm$ 0.4 & 61\% $\pm$ 1.8\% & 0.56 $\pm$ 0.004 & 74$^\circ$ $\pm$ 0.3$^\circ$ \\
spanning trees & Random & 104 $\pm$ 3 & 38\% $\pm$ 2.1\% & 2.4 $\pm$ 0.1 & 56\% $\pm$ 1.9\% & 0.51 $\pm$ 0.005 & 49$^\circ$ $\pm$ 0.6$^\circ$ \\
 & Random/All & & & 12.4 $\pm$ 0.3 & 56\% $\pm$ 1.9\% & & \\
\hline
Directed & All & 1000 & 41\% $\pm$ 1.2\% & 10.8 $\pm$ 0.2 & 45\% $\pm$ 1.5\% & 0.60 $\pm$ 0.020 & 61$^\circ$ $\pm$ 1.0$^\circ$ \\
acyclic graphs & Random & 386 $\pm$ 21 & 30\% $\pm$ 1.3\% & 8.5 $\pm$ 0.2 & 39\% $\pm$ 1.6\% & 0.40 $\pm$ 0.008 & 37$^\circ$ $\pm$ 1.0$^\circ$ \\
 & Random/All & & & 10.8 $\pm$ 0.2 & 39\% $\pm$ 1.6\% & & \\
\hline
Cardinality & All & 1000 & 42\% $\pm$ 1.4\% & 11.1 $\pm$ 0.4 & 45\% $\pm$ 1.8\% & 0.58 $\pm$ 0.011 & 65$^\circ$ $\pm$ 0.6$^\circ$ \\
constrained sets & Random & 272 $\pm$ 9 & 21\% $\pm$ 1.2\% & 6.0 $\pm$ 0.2 & 30\% $\pm$ 1.9\% & 0.44 $\pm$ 0.008 & 30$^\circ$ $\pm$ 0.8$^\circ$ \\
 & Random/All & & & 10.9 $\pm$ 0.3 & 29\% $\pm$ 2.1\% & & \\
\hline
\end{tabular}
\end{center}
\end{small}
\end{table*}

In this section, we provide experimental evidence on synthetic data.
Note that the work of \cite{Zhang14,Zhang15} has provided extensive experimental evidence on real-world datasets, for part-of-speech tagging and dependency parsing in the context of natural language processing.
Our experimental results are not only for directed spanning trees \cite{Zhang14,Zhang15} but also for directed acyclic graphs and cardinality-constrained sets.

We performed $30$ repetitions of the following procedure.
We generated a ground truth parameter ${w^*}$ with independent zero-mean and unit-variance Gaussian entries.
Then, we generated a training set $S$ of ${n=100}$ samples.
The fixed mapping $\phi$ from pairs ${(x,y)}$ to feature vectors ${\phi(x,y)}$ is as follows.
For every pair of possible edges/elements $i$ and $j$, we define ${\phi_{ij}(x,y) = \iverson{x_{ij}=1 \wedge i \in y \wedge j \in y}}$.
For instance, for directed spanning trees of $v$ nodes, we have ${x \in \{0,1\}^{\binom{v}{2}}}$ and ${\phi(x,y) \in \R^{\binom{v}{2}}}$.
In order to generate each training sample ${(x,y) \in S}$, we generated a random vector $x$ with independent Bernoulli entries, each with equal probability of being $1$ or $0$.
After generating $x$, we set ${y = f_{w^*}(x)}$.
That is, we solved eq.\eqref{eq:inferenceall} in order to produce the latent structured output $y$ from the observed input $x$ and the parameter ${w^*}$.

We compared two training methods: the maximum loss over all possible structured outputs as in eq.\eqref{eq:trainall}, and the maximum loss over random structured outputs as in eq.\eqref{eq:trainrandom}.
For both minimization problems, we replaced the \emph{discontinuous} 0/1 loss ${\iverson{z \geq 0}}$ with the \emph{convex} hinge loss ${\max{(0,1+z)}}$, as it is customary.
For both problems, we used ${\lambda = 1/n}$ as suggested by Theorems \ref{thm:pacbayesall} and \ref{thm:pacbayesrandom}, and we performed $20$ iterations of the subgradient descent method with a decaying step size ${1/\sqrt{t}}$ for iteration $t$.
For sampling random structured outputs in eq.\eqref{eq:trainrandom}, we implemented Algorithm \ref{alg:greedylocal} for directed spanning trees, directed acyclic graphs and cardinality-constrained sets.
We considered directed spanning trees of $6$ nodes, directed acyclic graphs of $5$ nodes and $2$ parents per node, and sets of $4$ elements chosen from $15$ possible elements.
We used ${\beta = 0.8}$ for directed spanning trees, ${\beta = 0.85}$ for directed acyclic graphs, and ${\beta = 0.5}$ for cardinality-constrained sets, as prescribed by Claims \ref{clm:trees}, \ref{clm:dags} and \ref{clm:cardsets}.
After training, for inference on an independent test set, we used eq.\eqref{eq:inferenceall} for the maximum loss over all possible structured outputs.
For the maximum loss over random structured outputs, we use the following \emph{approximate} inference approach:
\begin{align} \label{eq:inferencerandom}
\ft_w(x) \equiv \argmax_{y \in T(w,x)}{\dotprod{\phi(x,y)}{w}}
\end{align}

Table \ref{tab:results} shows the average over $30$ repetitions, and the standard error at 95\% confidence level of the following measurements.
We report the runtime, the training distortion as well as the test distortion in an independently generated set of $100$ samples.
We also report the normalized distance of the learnt $\wh$ to the ground truth ${w^*}$, i.e., ${\norm{\wh-w^*}_2/\sqrt{\ell}}$.
Additionally, we report the angle of the learnt $\wh$ with respect to the ground truth ${w^*}$, i.e. ${\arccos(\dotprod{\wh}{w^*}/(\norm{\wh}_2 \norm{w^*}_2))}$.
In the different study cases (directed spanning trees, directed acyclic graphs and cardinality-constrained sets), the maximum loss over random structured outputs outperforms the maximum loss over all possible structured outputs.

\section{Discussion} \label{sec:discussion}

In this section, we provide more details regarding the computational complexity of the inference problem.
We also present a brief review of the previous work and provide ideas for extending our theoretical result.

\paragraph{Computational Complexity of the Inference Problem.}

Very few cases of the general \emph{inference} problem in eq.\eqref{eq:inferenceall} are tractable.
For instance, if ${\Y(x)}$ is the set of directed spanning trees, and $w$ is a vector of edge weights (i.e., linear with respect to $y$), then eq.\eqref{eq:inferenceall} is equivalent to the maximum directed spanning tree problem, which is polynomial-time.
In general, the inference problem in eq.\eqref{eq:inferenceall} is not only NP-hard but also hard to approximate.
For instance, if ${\Y(x)}$ is the set of directed acyclic graphs, and $w$ is a vector of edge weights (i.e., linear with respect to $y$), then eq.\eqref{eq:inferenceall} is equivalent to the maximum acyclic subgraph problem, which approximating within a factor better than ${1/2}$ is unique-games hard \cite{Guruswami08}.
As an additional example, consider the case where ${\Y(x)}$ is the set of sets with up to some prespecified number of elements (i.e., ${\Y(x)}$ is a cardinality constraint), and the objective ${\dotprod{\phi(x,y)}{w}}$ is submodular with respect to $y$.
In this case, eq.\eqref{eq:inferenceall} cannot be approximated within a factor better than ${1-1/e}$ unless P=NP \cite{Nemhauser78}.

These negative results made us to avoid interpreting the maximum loss over random structured outputs in eq.\eqref{eq:trainrandom} as an approximate optimization algorithm for the maximum loss over all possible structured outputs in eq.\eqref{eq:trainall}.

\paragraph{Previous Work.}

Approximate inference was proposed in \cite{Kulesza07}, with an adaptation of the proof techniques in \cite{McAllester07}.
More specifically, \cite{Kulesza07} performs maximization of the loss over a \emph{superset} of feasible decodings of $x$, i.e., over ${y \in \Yp(x) \supseteq \Y(x)}$.
Note that our upper bound of the Gibbs decoder distortion dominates the maximum loss over ${y \in \Y(x)}$, and the latter dominates the upper bound of \cite{Kulesza07}.
One could potentially use a similar argument with respect to a \emph{subset} of feasible decodings of $x$, i.e., with respect to ${y \in \Yp(x) \subseteq \Y(x)}$.
Unfortunately, this approach does not obtain an upper bound of the Gibbs decoder distortion.

Tangential to our work, previous analyses have exclusively focused either on sample complexity or convergence.
Sample complexity analyses include margin bounds \cite{Taskar03}, Rademacher complexity \cite{London13} and PAC-Bayes bounds \cite{McAllester07,McAllester11}.
Convergence have been analyzed for specific algorithms for the separable \cite{Collins04} and nonseparable \cite{Crammer06} cases.

\paragraph{Concluding Remarks.}

The work of \cite{Zhang14,Zhang15} has shown extensive experimental evidence for part-of-speech tagging and dependency parsing in the context of natural language processing.
In this paper, we present a theoretical analysis that explains the experimental success of \cite{Zhang14,Zhang15} for directed spanning trees.
Our analysis was provided for a far more general setup, which allowed proposing algorithms for other types of structured outputs, such as directed acyclic graphs and cardinality-constrained sets.
We hope that our theoretical work will motivate experimental validation on many other real-world structured prediction problems.

There are several ways of extending this research.
While we focused on Gaussian perturbations, it would be interesting to analyze other distributions from the computational as well as statistical viewpoints.
We analyzed a general class of proposal distributions that depend on the induced linear orderings.
Algorithms that make greedy local changes, traverse the set of feasible decodings in a constrained fashion, by following allowed moves defined by some prespecified graph.
The addition of these graph-theoretical constraints would enable obtaining tighter upper bounds.
From a broader perspective, extensions of our work to latent models \cite{Ping14,Yu09} as well as maximum a-posteriori perturbation models \cite{Gane14,Papandreou11} would be of great interest.
Finally, while we focused on learning the parameter of structured prediction models, it would be interesting to analyze \emph{approximate} inference for prediction on an independent test set.

\bibliographystyle{mlapa}
\bibliography{references}

\clearpage
\onecolumn

\rule{1\linewidth}{1mm}

\begin{center}
{\Large\textbf{SUPPLEMENTARY MATERIAL.\\
Structured Prediction: From Gaussian Perturbations\\
to Linear-Time Principled Algorithms}}
\end{center}

\rule{1\linewidth}{.3mm}

\appendix

\section{Detailed Proofs} \label{sec:detailedproofs}

In this section, we state the proofs of all the theorems and claims in our manuscript.

\subsection{Proof of Theorem \ref{thm:pacbayesall}}

Here, we provide the proof of Theorem \ref{thm:pacbayesall}.
First, we derive an intermediate lemma needed for the final proof.

\begin{lemma}[Adapted\footnote{
We make two small corrections to Lemma~6 of \cite{McAllester07}.
First, it is only stated for ${\yp = f_w(x)}$ but it does not make use of the optimality of ${f_w(x)}$, thus, it holds for any ${\yp \in \Y(x)}$.
Second, for the union bound over all ${p \in \cup_{(x,y) \in S}{\PS(x)}}$, we assume that ${|\cup_{(x,y) \in S}{\PS(x)}| \leq \ell}$.
Instead, Lemma~6 in \cite{McAllester07} incorrectly assumes ${|\PS(x)| \leq \ell}$ for all ${x \in \X}$, and thus ${|\cup_{(x,y) \in S}{\PS(x)}| \leq \sum_{(x,y) \in S}{|\PS(x)|} \leq n \ell}$.
} from Lemma~6 in \npcite{McAllester07}] \label{lem:gaussian}
Assume that there exists a finite integer value $\ell$ such that ${|\cup_{(x,y) \in S}{\PS(x)}| \leq \ell}$.
Let ${Q(w)}$ be a unit-variance Gaussian distribution centered at ${\alpha w}$ for ${\alpha = \sqrt{2\log{(2 n \ell/\norm{w}_2^2)}}}$.
Simultaneously for all ${(x,y) \in S}$, ${\yp \in \Y(x)}$ and ${w \in \W}$, we have:
\begin{align*}
\P_{\wp \sim Q(w)}[H(x,\yp,f_{\wp}(x)) - m(x,\yp,f_{\wp}(x),w) < 0] \leq \norm{w}_2^2/n
\end{align*}
\noindent or equivalently:
\begin{align} \label{eq:gaussianwhp}
\P_{\wp \sim Q(w)}[H(x,\yp,f_{\wp}(x)) - m(x,\yp,f_{\wp}(x),w) \geq 0] \geq 1 - \norm{w}_2^2/n
\end{align}
\end{lemma}
\begin{proof}
First, note that ${\wp - \alpha w}$ is a zero-mean and unit-variance Gaussian random vector.
By well-known Gaussian concentration inequalities, for any ${p \in \PS(x)}$ we have:
\begin{align*}
\P_{\wp \sim Q(w)}[|\wp_p - \alpha w_p| \geq \eps] \leq 2 \exp{-\eps^2/2}
\end{align*}
By the union bound and setting ${\eps = \alpha = \sqrt{2\log{(2 n \ell/\norm{w}_2^2)}}}$, we have:
\begin{align*}
\P_{\wp \sim Q(w)}[(\exists p \in \cup_{(x,y) \in S}{\PS(x)}) {\rm\ } |\wp_p - \alpha w_p| \geq \alpha] & \leq 2 |\cup_{(x,y) \in S}{\PS(x)}| \exp{-\alpha^2/2} \\
 & = |\cup_{(x,y) \in S}{\PS(x)}| \frac{\norm{w}_2^2}{\ell n} \\
 & \leq \norm{w}_2^2/n
\end{align*}
\noindent or equivalently:
\begin{align*}
\P_{\wp \sim Q(w)}[(\forall p \in \cup_{(x,y) \in S}{\PS(x)}) {\rm\ } |\wp_p - \alpha w_p| < \alpha] \geq 1 - \norm{w}_2^2/n
\end{align*}
The high-probability statement in eq.\eqref{eq:gaussianwhp} can be written as:
\begin{align*}
\yh = f_{\wp}(x) {\rm\ } \Rightarrow {\rm\ } H(x,\yp,\yh) - m(x,\yp,\yh,w) \geq 0
\end{align*}
Next, we use proof by contradiction, i.e., we will assume:
\begin{align*}
\yh = f_{\wp}(x) {\rm\ } \text{and} {\rm\ } H(x,\yp,\yh) - m(x,\yp,\yh,w) < 0
\end{align*}
\noindent and arrive to a contradiction ${\yh \neq f_{\wp}(x)}$.
From the above, we have:
\begin{align*}
m(x,\yp,\yh,\wp) & = m(x,\yp,\yh,\alpha w + (\wp - \alpha w)) \\
 & = \alpha m(x,\yp,\yh,w) - \dotprod{(\phi(x,\yp) - \phi(x,\yh))}{(\alpha w - \wp)} \\
 & > \alpha H(x,\yp,\yh) - \dotprod{(\phi(x,\yp) - \phi(x,\yh))}{(\alpha w - \wp)} \\
 & = \alpha H(x,\yp,\yh) - \sum_{p \in \PS(x)}{ (c(p,x,\yp) - c(p,x,\yh)) (\alpha w_p - \wp_p) } \\
 & \geq \alpha H(x,\yp,\yh) - \sum_{p \in \PS(x)}{ |c(p,x,\yp) - c(p,x,\yh)| |\alpha w_p - \wp_p| } \\
 & \geq \alpha H(x,\yp,\yh) - \sum_{p \in \PS(x)}{ |c(p,x,\yp) - c(p,x,\yh)| } \alpha \\
 & = 0
\end{align*}
Note that ${m(x,\yp,\yh,\wp) > 0}$ if and only if ${\dotprod{\phi(x,\yp)}{w} > \dotprod{\phi(x,\yh)}{w}}$.
Therefore ${\yh \neq f_{\wp}(x)}$ since it does not maximize ${\dotprod{\phi(x,\cdot)}{w}}$ as defined in eq.\eqref{eq:inferenceall}.
Thus, we prove our claim.
\qedhere
\end{proof}

Next, we provide the final proof.

\begin{proof}[Proof of Theorem \ref{thm:pacbayesall}]
Define the Gibbs decoder \emph{empirical} distortion of the perturbation distribution ${Q(w)}$ and training set $S$ as:
\begin{align*}
L(Q(w),S) = \frac{1}{n} \sum_{(x,y) \in S}{\E_{\wp \sim Q(w)}[d(y,f_{\wp}(x))]}
\end{align*}
In PAC-Bayes terminology, ${Q(w)}$ is the \emph{posterior} distribution.
Let the \emph{prior} distribution $P$ be the unit-variance zero-mean Gaussian distribution.
Fix ${\delta \in (0,1)}$ and ${\alpha>0}$.
By well-known PAC-Bayes proof techniques, Lemma~4 in \cite{McAllester07} shows that with probability at least ${1-\delta/2}$ over the choice of $n$ training samples, simultaneously for all parameters ${w \in \W}$, and unit-variance Gaussian posterior distributions ${Q(w)}$ centered at ${w \alpha}$, we have:
\begin{align} \label{eq:pacbayes}
L(Q(w),D) & \leq L(Q(w),S) + \sqrt{\frac{KL(Q(w) \| P) + \log{(2n/\delta)}}{2(n-1)}} \nonumber \\
 & = L(Q(w),S) + \sqrt{\frac{\norm{w}_2^2 \alpha^2/2 + \log{(2n/\delta)}}{2(n-1)}}
\end{align}
Thus, an upper bound of ${L(Q(w),S)}$ would lead to an upper bound of ${L(Q(w),D)}$.
In order to upper-bound ${L(Q(w),S)}$, we can upper-bound each of its summands, i.e., we can upper-bound ${\E_{\wp \sim Q(w)}[d(y,f_{\wp}(x))]}$ for each ${(x,y) \in S}$.
Define the distribution ${Q(w,x)}$ with support on ${\Y(x)}$ in the following form for all ${y \in \Y(x)}$:
\begin{align} \label{eq:Qwx}
\P_{\yp \sim Q(w,x)}[\yp = y] \equiv \P_{\wp \sim Q(w)}[f_{\wp}(x) = y]
\end{align}
For clarity of presentation, define:
\begin{align*}
u(x,y,\yp,w) \equiv H(x,y,\yp) - m(x,y,\yp,w)
\end{align*}
Let ${u \equiv u(x,y,f_{\wp}(x),w)}$.
Simultaneously for all ${(x,y) \in S}$, we have:
\begin{align} \refstepcounter{equation}
\E_{\wp \sim Q(w)}[d(y,f_{\wp}(x)] & = \E_{\wp \sim Q(w)}[d(y,f_{\wp}(x)) {\rm\ } \iverson{u \geq 0} + d(y,f_{\wp}(x)) {\rm\ } \iverson{u < 0}] \nonumber \\
 & \leq \E_{\wp \sim Q(w)}[d(y,f_{\wp}(x)) {\rm\ } \iverson{u \geq 0} + \iverson{u < 0}] \tag{\theequation.a}\label{eq:dtimesiversontoiverson} \\
 & = \E_{\wp \sim Q(w)}[d(y,f_{\wp}(x)) {\rm\ } \iverson{u \geq 0}] + \P_{\wp \sim Q(w)}[u < 0] \nonumber \\
 & \leq \E_{\wp \sim Q(w)}[d(y,f_{\wp}(x)) {\rm\ } \iverson{u \geq 0}] + \norm{w}_2^2/n \tag{\theequation.b}\label{eq:highprobHm} \\
 & = \E_{\wp \sim Q(w)}[d(y,f_{\wp}(x)) {\rm\ } \iverson{u(x,y,f_{\wp}(x),w) \geq 0}] + \norm{w}_2^2/n \nonumber \\
 & = \E_{\yp \sim Q(w,x)}[d(y,\yp) {\rm\ } \iverson{u(x,y,\yp,w) \geq 0}] + \norm{w}_2^2/n \tag{\theequation.c}\label{eq:QwtoQwx} \\
 & \leq \max_{\yh \in \Y(x)}{d(y,\yh) {\rm\ } \iverson{u(x,y,\yh,w) \geq 0}} + \norm{w}_2^2/n \tag{\theequation.d}\label{eq:expectedtomax}
\end{align}
\noindent where the step in eq.\eqref{eq:dtimesiversontoiverson} holds since ${d : \Y \times \Y \to [0,1]}$.
The step in eq.\eqref{eq:highprobHm} follows from Lemma \ref{lem:gaussian} which states that ${\P_{\wp \sim Q(w)}[u(x,\yp,f_{\wp}(x),w) < 0] \leq \norm{w}_2^2/n}$ for ${\alpha = \sqrt{2\log{(2 n \ell/\norm{w}_2^2)}}}$, simultaneously for all ${(x,y) \in S}$, ${\yp \in \Y(x)}$ and ${w \in \W}$.
By the definition in eq.\eqref{eq:Qwx}, then the step in eq.\eqref{eq:QwtoQwx} holds.
Let ${g : \Y \to [0,1]}$ be some arbitrary function, the step in eq.\eqref{eq:expectedtomax} uses the fact that ${\E_{y}[g(y)] \leq \max_{y}{g(y)}}$.

By eq.\eqref{eq:pacbayes} and eq.\eqref{eq:expectedtomax}, we prove our claim.
\qedhere
\end{proof}

\subsection{Proof of Theorem \ref{thm:pacbayesrandom}}

Here, we provide the proof of Theorem \ref{thm:pacbayesrandom}.
First, we derive an intermediate lemma needed for the final proof.

\begin{lemma} \label{lem:phmbound}
Let ${\Delta \in \R^k}$ be a random variable, and ${w \in \R^k}$ be a constant.
If ${\dotprod{\E[\mu(\Delta)]}{w} \leq 1/2}$ then we have:
\begin{align*}
\P[\norm{\Delta}_1 - \dotprod{\Delta}{w} < 0] \leq \lexp{\frac{-1}{32 \norm{w}_2^2}}
\end{align*}
\end{lemma}
\begin{proof}
Let ${t > 0}$, we have that:
\begin{align} \refstepcounter{equation}
\P[\norm{\Delta}_1 - \dotprod{\Delta}{w} < 0] & = \P[\dotprod{\mu(\Delta)}{w} > 1] \tag{\theequation.a}\label{eq:useunit} \\
 & = \P[\dotprod{(\mu(\Delta) - \E[\mu(\Delta)])}{w} > 1 - \dotprod{\E[\mu(\Delta)]}{w}] \nonumber \\
 & \leq \P[\dotprod{(\mu(\Delta) - \E[\mu(\Delta)])}{w} \geq 1/2] \tag{\theequation.b}\label{eq:uselownorm} \\
 & = \P[\lexp{t \dotprod{(\mu(\Delta) - \E[\mu(\Delta)])}{w}} \geq \exp{t/2}] \nonumber \\
 & \leq \exp{-t/2} \, \E[\lexp{t \dotprod{(\mu(\Delta) - \E[\mu(\Delta)])}{w}}] \tag{\theequation.c}\label{eq:usemarkov} \\
 & \leq \lexp{-t/2 + 2 t^2 \norm{w}_2^2} \tag{\theequation.d}\label{eq:usehoeffding}
\end{align}
\noindent where the step in eq.\eqref{eq:useunit} follows from dividing ${\norm{\Delta}_1 - \dotprod{\Delta}{w}}$ by ${\norm{\Delta}_1}$.
Note that ${\Delta = 0}$ does not fulfill either of the two expressions ${\norm{\Delta}_1 - \dotprod{\Delta}{w} < 0}$, or ${\dotprod{\mu(\Delta)}{w} > 1}$.
The step in eq.\eqref{eq:uselownorm} follows from ${\dotprod{\E[\mu(\Delta)]}{w} \leq 1/2}$ and thus ${1-\dotprod{\E[\mu(\Delta)]}{w} \geq 1/2}$.
The step in eq.\eqref{eq:usemarkov} follows from Markov's inequality.
The step in eq.\eqref{eq:usehoeffding} follows from Hoeffding's lemma and the fact that the random variable ${z = \dotprod{(\mu(\Delta) - \E[\mu(\Delta)])}{w}}$ fulfills ${\E[z] = 0}$ as well as ${z \in [-2 \norm{w}_2, +2 \norm{w}_2]}$.
In more detail, note that ${\norm{\mu(\Delta)}_2 \leq 1}$ since it holds trivially for ${\Delta=0}$, and for ${\Delta \neq 0}$ we have that ${\norm{\mu(\Delta)}_2 = \norm{\Delta}_2/\norm{\Delta}_1 \leq 1}$.
By Jensen's inequality ${\norm{\E[\mu(\Delta)]}_2 \leq \E[\norm{\mu(\Delta)}_2] \leq 1}$.
Then, note that by Cauchy-Schwarz inequality ${|\dotprod{(\mu(\Delta) - \E[\mu(\Delta)])}{w}| \leq \norm{\mu(\Delta) - \E[\mu(\Delta)]}_2 \norm{w}_2 \leq (\norm{\mu(\Delta)}_2 + \norm{\E[\mu(\Delta)]}_2) \norm{w}_2 \leq 2 \norm{w}_2}$.
Finally, let ${g(t) = -t/2 + 2 t^2 \norm{w}_2^2}$.
By making ${\partial g / \partial t = 0}$, we get the optimal setting ${t^* = 1/(8 \norm{w}_2^2)}$.
Thus, ${g(t^*) = -1/(32 \norm{w}_2^2)}$ and we prove our claim.
\qedhere
\end{proof}

Next, we provide the final proof.

\begin{proof}[Proof of Theorem \ref{thm:pacbayesrandom}]
Note that sampling from the distribution ${Q(w,x)}$ as defined in eq.\eqref{eq:Qwx} is NP-hard in general, thus our plan is to upper-bound the expectation in eq.\eqref{eq:QwtoQwx} by using the maximum over random structured outputs sampled independently from a proposal distribution ${R(w,x)}$ with support on ${\Y(x)}$.

Let ${T(w,x)}$ be a set of $\m$ i.i.d. random structured outputs drawn from the proposal distribution ${R(w,x)}$, i.e., ${T(w,x) \sim R(w,x)^{\m}}$.
Furthermore, let $\T(w)$ be the collection of the $n$ sets ${T(w,x)}$ for all ${(x,y) \in S}$, i.e. ${\T(w) \equiv \{T(w,x)\}_{(x,y) \in S}}$ and thus ${\T(w) \sim \{R(w,x)^{\m}\}_{(x,y) \in S}}$.
For clarity of presentation, define:
\begin{align*}
v(x,y,\yp,w) \equiv d(y,\yp) {\rm\ } \iverson{H(x,y,\yp) - m(x,y,\yp,w) \geq 0}
\end{align*}
For sets ${T(w,x)}$ of sufficient size $\m$, our goal is to upper-bound eq.\eqref{eq:QwtoQwx} in the following form for all parameters ${w \in \W}$:
\begin{align*}
\frac{1}{n} \sum_{(x,y) \in S}{ \E_{\yp \sim Q(w,x)}[v(x,y,\yp,w)] } \leq \frac{1}{n} \sum_{(x,y) \in S}{ \max_{\yh \in T(w,x)}{v(x,y,\yh,w)} } + \O(\sfrac{\log^{3/2}{n}}{\sqrt{n}})
\end{align*}
Note that the above expression would produce a tighter upper bound than the maximum loss over all possible structured outputs since ${\max_{\yh \in T(w,x)}{v(x,y,\yh,w)} \leq \max_{\yh \in \Y(x)}}{v(x,y,\yh,w)}$.
For analysis purposes, we decompose the latter equation into two quantities:
\begin{align}
A(w,S) & \equiv \frac{1}{n} \sum_{(x,y) \in S}{\left( \E_{\yp \sim Q(w,x)}[v(x,y,\yp,w)] - \E_{T(w,x) \sim R(w,x)^{\m}}\left[\max_{\yh \in T(w,x)}{v(x,y,\yh,w)}\right] \right)} \label{eq:A} \\
B(w,S,\T(w)) & \equiv \frac{1}{n} \sum_{(x,y) \in S}{\left( \E_{T(w,x) \sim R(w,x)^{\m}}\left[\max_{\yh \in T(w,x)}{v(x,y,\yh,w)}\right] - \max_{\yh \in T(w,x)}{v(x,y,\yh,w)} \right)} \label{eq:B}
\end{align}
Thus, we will show that ${A(w,S) \leq \sqrt{\sfrac{1}{n}}}$ and ${B(w,S,\T(w)) \leq \O(\sfrac{\log^{3/2}{n}}{\sqrt{n}})}$ for all parameters ${w \in \W}$, any training set $S$ and all collections $\T(w)$, and therefore ${A(w,S) + B(w,S,\T(w)) \leq \O(\sfrac{\log^{3/2}{n}}{\sqrt{n}})}$.
Note that while the value of ${A(w,S)}$ is deterministic, the value of ${B(w,S,\T(w))}$ is stochastic given that $\T(w)$ is a collection of sampled random structured outputs.

Fix a specific ${w \in \W}$.
If data is separable then ${v(x,y,\yp,w)=0}$ for all ${(x,y) \in S}$ and ${\yp \in \Y(x)}$.
Thus, we have ${A(w,S)=B(w,S,\T(w)) = 0}$ and we complete our proof for the separable case.\footnote{
The same result can be obtained for any subset of $S$ for which the ``separability'' condition holds.
Therefore, our analysis with the ``nonseparability'' condition can be seen as a worst case scenario.}
In what follows, we focus on the nonseparable case.

\paragraph{Bounding the Deterministic Expectation ${A(w,S)}$.}

Here, we show that in eq.\eqref{eq:A}, ${A(w,S) \leq \sqrt{\sfrac{1}{n}}}$ for all parameters ${w \in \W}$ and any training set $S$, provided that we use a sufficient number $\m$ of random structured outputs sampled from the proposal distribution.

By well-known identities, we can rewrite:
\begin{align} \refstepcounter{equation}
A(w,S) & = \frac{1}{n} \sum_{(x,y) \in S}{\int_0^1{\left( \P_{\yp \sim R(w,x)}[v(x,y,\yp,w) \leq z]^{\m} - \P_{\yp \sim Q(w,x)}[v(x,y,\yp,w) \leq z] \right)}dz} \tag{\theequation.a}\label{eq:expectedtoprob} \\
 & \leq \frac{1}{n} \sum_{(x,y) \in S}{ \P_{\yp \sim R(w,x)}[v(x,y,\yp,w) < 1]^{\m} } \nonumber \\
 & = \frac{1}{n} \sum_{(x,y) \in S}{ \P_{\yp \sim R(w,x)}[d(y,\yp) < 1 \vee H(x,y,\yp) - m(x,y,\yp,w) < 0]^{\m} } \nonumber \\
 & = \frac{1}{n} \sum_{(x,y) \in S}{ \left(1 - \P_{\yp \sim R(w,x)}[d(y,\yp) = 1 \wedge H(x,y,\yp) - m(x,y,\yp,w) \geq 0]\right)^{\m} } \nonumber \\
 & \leq \frac{1}{n} \sum_{(x,y) \in S}{ \left( 1 - \min{\left( \P_{\yp \sim R(w,x)}[d(y,\yp) = 1] {\rm\ ,\ } \P_{\yp \sim R(w,x)}[H(x,y,\yp) - m(x,y,\yp,w) \geq 0] \right)} \right)^{\m} } \nonumber \\
 & = \frac{1}{n} \sum_{(x,y) \in S}{ \max{\left( 1 - \P_{\yp \sim R(w,x)}[d(y,\yp) = 1] {\rm\ ,\ } \P_{\yp \sim R(w,x)}[H(x,y,\yp) - m(x,y,\yp,w) < 0] \right)}^{\m} } \nonumber \\
 & \leq \max{\left( \beta {\rm\ ,\ } \lexp{\frac{-1}{32 \norm{w}_2^2}} \right)}^{\m} \tag{\theequation.b}\label{eq:betaphmbound} \\
 & = \sqrt{1/n} \tag{\theequation.c}\label{eq:Abound}
\end{align}
\noindent where the step in eq.\eqref{eq:expectedtoprob} holds since for two independent random variables ${g,h \in [0,1]}$, we have ${\E[g] = 1-\int_0^1{\P[g \leq z] dz}}$ and ${\P[\max{(g,h)} \leq z] = \P[g \leq z]\P[h \leq z]}$.
Therefore, $\E[\max{(g,h)}] = 1-\int_0^1{\P[g \leq z]\P[h \leq z] dz}$.
For the step in eq.\eqref{eq:betaphmbound}, we used Assumption \ref{asm:maxdistortion} for the first term in the $\max$.
For the second term in the $\max$, we used Assumption \ref{asm:lownorm}.
More formally, let ${\Delta \equiv \phi(x,y)-\phi(x,\yp)}$ then ${H(x,y,\yp)=\norm{\Delta}_1}$ and ${m(x,y,\yp,w) = \dotprod{\Delta}{w}}$.
By Assumption \ref{asm:lownorm}, we have that ${\norm{\E[\mu(\Delta)]}_2 \leq 1/(2\sqrt{n}) \leq 1/(2 \norm{w}_2)}$.
By Cauchy-Schwarz inequality we have ${\dotprod{\E[\mu(\Delta)]}{w} \leq \norm{\E[\mu(\Delta)]}_2 \norm{w}_2 \leq \norm{w}_2/(2 \norm{w}_2) \leq 1/2}$.
Since ${\dotprod{\E[\mu(\Delta)]}{w} \leq 1/2}$, we apply Lemma \ref{lem:phmbound} in the step in eq.\eqref{eq:betaphmbound}.
For the step in eq.\eqref{eq:Abound}, let ${\alpha \equiv \max{\left(\frac{1}{\log{(1/\beta)}}, 32 \norm{w}_2^2\right)}}$.
Note that ${\max{\left( \beta {\rm\ ,\ } \lexp{\frac{-1}{32 \norm{w}_2^2}} \right)} = \exp{-1/\alpha}}$.
Furthermore, let ${\m = \frac{1}{2} \alpha \log{n}}$.
Therefore, ${\max{\left( \beta {\rm\ ,\ } \lexp{\frac{-1}{32 \norm{w}_2^2}} \right)}^{\m} = (\exp{-1/\alpha})^{\frac{1}{2} \alpha \log{n}} = \exp{\frac{-1}{2} \log{n}} = \sqrt{1/n}}$.

\paragraph{Bounding the Stochastic Quantity ${B(w,S,\T(w))}$.}

Here, we show that in eq.\eqref{eq:B}, ${B(w,S,\T(w)) \leq \O(\sfrac{\log^{3/2}{n}}{\sqrt{n}})}$ for all parameters ${w \in \W}$, any training set $S$ and all collections $\T(w)$.
For clarity of presentation, define:
\begin{align*}
g(x,y,T,w) \equiv \max_{\yh \in T}{v(x,y,\yh,w)}
\end{align*}
Thus, we can rewrite:
\begin{align*}
B(w,S,\T(w)) = \frac{1}{n} \sum_{(x,y) \in S}{\left( \E_{T(w,x) \sim R(w,x)^{\m}}[g(x,y,T(w,x),w)] - g(x,y,T(w,x),w) \right)}
\end{align*}
Let ${r(x) \equiv |\Y(x)|}$ and thus ${\Y(x) \equiv \{y_1 \dots y_{r(x)}\}}$.
Let ${\pi(x) = (\pi_1 \dots \pi_{r(x)})}$ be a permutation of ${\{1 \dots r(x)\}}$ such that ${\dotprod{\phi(x,y_{\pi_1})}{w} < \dots < \dotprod{\phi(x,y_{\pi_{r(x)}})}{w}}$.
Let $\Pi$ be the collection of the $n$ permutations ${\pi(x)}$ for all ${(x,y) \in S}$, i.e. ${\Pi = \{\pi(x)\}_{(x,y) \in S}}$.
From Assumption \ref{asm:linearordering}, we have that ${R(\pi(x),x) \equiv R(w,x)}$.
Similarly, we rewrite ${T(\pi(x),x) \equiv T(w,x)}$ and ${\T(\Pi) \equiv \T(w)}$.

Furthermore, let ${\W_{\Pi,S}}$ be the set of all ${w \in \W}$ that induce $\Pi$ on the training set $S$.
For the parameter space $\W$, collection $\Pi$ and training set $S$, define the function class ${\G_{\W,\Pi,S}}$ as follows:
\begin{align*}
\G_{\W,\Pi,S} \equiv \{g(x,y,T,w) \mid w \in \W_{\Pi,S} \wedge (x,y) \in S \}
\end{align*}
Note that since ${|\Y(x)| \leq r}$ for all ${(x,y) \in S}$, then ${|\cup_{(x,y) \in S}{\Y(x)}| \leq \sum_{(x,y) \in S}|\Y(x)| \leq nr}$.
Note that each ordering of the $nr$ structured outputs completely determines a collection $\Pi$ and thus the collection of proposal distributions ${R(w,x)}$ for each ${(x,y) \in S}$.
Note that since ${|\cup_{(x,y) \in S}{\PS(x)}| \leq \ell}$, we need to consider ${\phi(x,y) \in \R^\ell}$.
Although we can consider ${w \in \R^\ell}$, the vector $w$ is sparse with at most $\s$ non-zero entries.
Thus, we take into account all possible subsets of $\s$ features from $\ell$ possible features.
From results in \cite{Bennett56,Bennett60,Cover67}, we can conclude that there are at most ${(nr)^{2(\s-1)}}$ linearly inducible orderings, for a fixed set of $\s$ features.
Therefore, there are at most ${\binom{\ell}\s (nr)^{2(\s-1)} \leq \ell^\s (nr)^{2 \s}}$ collections $\Pi$.

Fix ${\delta \in (0,1)}$.
By Rademacher-based uniform convergence\footnote{
Note that for the analysis of ${B(w,S,\T(w))}$, the training set $S$ is fixed and randomness stems from the collection ${\T(w)}$.
Also, note that for applying McDiarmid's inequality, independence of each set ${T(w,x)}$ for all ${(x,y) \in S}$ is a sufficient condition, and identically distributed sets ${T(w,x)}$ are not necessary.
} and by a union bound over all ${\ell^\s (nr)^{2 \s}}$ collections $\Pi$, with probability at least ${1-\delta/2}$ over the choice of $n$ sets of random structured outputs, simultaneously for all parameters ${w \in \W}$:
\begin{align} \label{eq:Bbound}
B(w,S,\T(w)) \leq 2 {\rm\ } \Rademacher_{\T(\Pi)}(\G_{\W,\Pi,S}) + 3 \sqrt{\frac{\s (\log{\ell}+2 \log{(nr)})+\log{(4/\delta)}}{n}}
\end{align}
\noindent where ${\Rademacher_{\T(\Pi)}(\G_{\W,\Pi,S})}$ is the \emph{empirical} Rademacher complexity of the function class ${\G_{\W,\Pi,S}}$ with respect to the collection ${\T(\Pi)}$ of the $n$ sets ${T(\pi(x),x)}$ for all ${(x,y) \in S}$.
For clarity, define:
\begin{align*}
\Delta_p(x,y,\yp) \equiv \begin{cases}
c(p,x,y) - c(p,x,\yp) & \text{if } p \in \PS(x) \\
0 & \text{otherwise}
\end{cases}
\end{align*}
Let $\sigma$ be an $n$-dimensional vector of independent Rademacher random variables indexed by ${(x,y) \in S}$, i.e., ${\P[\sigma_{(x,y)}=+1]=\P[\sigma_{(x,y)}=-1]=1/2}$.
The empirical Rademacher complexity is defined as:
\begin{align} \refstepcounter{equation}
\Rademacher_{\T(\Pi)}(\G_{\W,\Pi,S}) & \equiv \E_\sigma\left[ \sup_{g \in \G_{\W,\Pi,S}}{\left( \frac{1}{n} \sum_{(x,y) \in S}{\sigma_{(x,y)} g(x,y,T(\pi(x),x),w)} \right)} \right] \nonumber \\
 & = \E_\sigma\left[ \sup_{w \in \W_{\Pi,S}}{\left( \frac{1}{n} \sum_{(x,y) \in S}{\sigma_{(x,y)} \max_{\yh \in T(\pi(x),x)}{d(y,\yh) {\rm\ } \iverson{H(x,y,\yh) - m(x,y,\yh,w) \geq 0}}} \right)} \right] \nonumber \\
 & = \E_\sigma\left[ \sup_{w \in \W_{\Pi,S}}{\left( \frac{1}{n} \sum_{(x,y) \in S}{\sigma_{(x,y)} \max_{\yh \in T(\pi(x),x)}{d(y,\yh) {\rm\ } \iverson{\norm{\Delta(x,y,\yh)}_1 - \dotprod{\Delta(x,y,\yh)}{w} \geq 0}}} \right)} \right] \nonumber \\
 & = \E_\sigma\left[ \sup_{w \in \R^\ell \setminus \{0\}}{\left( \frac{1}{n} \sum_{i \in \{1 \dots n\}}{\sigma_i \max_{j \in \{1 \dots \m\}}{d_{ij} {\rm\ } \iverson{\norm{z_{ij}}_1 - \dotprod{z_{ij}}{w} \geq 0}}} \right)} \right] \tag{\theequation.a}\label{eq:maxiversonnormlinear} \\
 & \leq \sum_{j \in \{1 \dots \m\}}{ \E_\sigma\left[ \sup_{w \in \R^\ell \setminus \{0\}}{\left( \frac{1}{n} \sum_{i \in \{1 \dots n\}}{\sigma_i {\rm\ } d_{ij} {\rm\ } \iverson{\norm{z_{ij}}_1 - \dotprod{z_{ij}}{w} \geq 0}} \right)} \right] } \tag{\theequation.b}\label{eq:maxtosum} \\
 & \leq \sum_{j \in \{1 \dots \m\}}{ \E_\sigma\left[ \sup_{w \in \R^\ell \setminus \{0\}}{\left( \frac{1}{n} \sum_{i \in \{1 \dots n\}}{\sigma_i {\rm\ } \iverson{\norm{z_{ij}}_1 - \dotprod{z_{ij}}{w} \geq 0}} \right)} \right] } \tag{\theequation.c}\label{eq:composition} \\
 & \leq \sum_{j \in \{1 \dots \m\}}{ \E_\sigma\left[ \sup_{w \in \R^{\ell+1} \setminus \{0\}}{\left( \frac{1}{n} \sum_{i \in \{1 \dots n\}}{\sigma_i {\rm\ } \iverson{\dotprod{z_{ij}}{w} \geq 0}} \right)} \right] } \tag{\theequation.d}\label{eq:removenorm} \\
 & \leq 2 \m \sqrt{\frac{\s \log{(\ell+1)} \log{(n+1)}}{n}} \tag{\theequation.e}\label{eq:rademacher}
\end{align}
\noindent where in the step in eq.\eqref{eq:maxiversonnormlinear}, the terms ${\sigma_i}$, ${d_{ij}}$ and ${z_{ij}}$ correspond to ${\sigma_{(x,y)}}$, ${d(y,\yh)}$ and ${\Delta(x,y,\yh)}$ respectively.
Thus, we assume that index $i$ corresponds to the training sample ${(x,y) \in S}$, and that index $j$ corresponds to the structured output ${\yh \in T(\pi(x),x)}$.
Note that since ${|\cup_{(x,y) \in S}{\PS(x)}| \leq \ell}$, thus the step in eq.\eqref{eq:maxiversonnormlinear} considers ${w,z_{ij} \in \R^\ell \setminus \{0\}}$ without loss of generality.
The step in eq.\eqref{eq:maxtosum} follows from the fact that for any two function classes $\G$ and $\H$, we have that ${\Rademacher(\{\max{(g,h)} \mid g \in \G \wedge h \in \H \}) \leq \Rademacher(\G) + \Rademacher(\H)}$.
The step in eq.\eqref{eq:composition} follows from the composition lemma and the fact that ${d_{ij} \in [0,1]}$ for all $i$ and $j$.
The step in eq.\eqref{eq:removenorm} considers a larger function class, since the value of ${\norm{z_{ij}}_1}$ can be taken as an additional entry in the vector ${z_{ij}}$ we consider ${w,z_{ij} \in \R^{\ell+1} \setminus \{0\}}$.
The step in eq.\eqref{eq:rademacher} follows from the Massart lemma, the Sauer-Shelah lemma and the VC-dimension of sparse linear classifiers.
That is, for any function class $\G$, we have that ${\Rademacher(\G) \leq \sqrt{\frac{2 VC(\G) \log{(n+1)}}{n}}}$ where ${VC(\G)}$ is the VC-dimension of $\G$.
Furthermore, by Theorem~20 of \cite{Neylon06}, ${VC(\G) \leq 2 \s \log{(\ell+1)}}$ for the class $\G$ of sparse linear classifiers on ${\R^{\ell+1}}$, with ${3 \leq \s \leq \frac{9}{20} \sqrt{\ell+1}}$.

By eq.\eqref{eq:pacbayes}, eq.\eqref{eq:QwtoQwx}, eq.\eqref{eq:Abound}, eq.\eqref{eq:Bbound} and eq.\eqref{eq:rademacher}, we prove our claim.
\qedhere
\end{proof}

\subsection{Proof of Claim \ref{clm:changeofmeasure}}

\begin{proof}
 For all ${(x,y) \in S}$ and ${w \in \W}$, by definition of the total variation distance, we have for any event ${\A(x,y,\yp,w)}$:
\begin{align*}
\left|\P_{\yp \sim R(w,x)}[\A(x,y,\yp,w)] - \P_{\yp \sim \Rp(w,x)}[\A(x,y,\yp,w)]\right| & \leq TV(R(w,x) \| \Rp(w,x))
\end{align*}
Let the event ${\A(x,y,\yp,w) : d(y,\yp) = 1 \wedge H(x,y,\yp) - m(x,y,\yp,w) \geq 0}$.
Since ${R(w,x)}$ fulfills Assumption \ref{asm:maxdistortion} with value ${\beta_1}$ and since ${TV(R(w,x) \| \Rp(w,x)) \leq \beta_2}$, we have that for all ${(x,y) \in S}$ and ${w \in \W}$:
\begin{align*}
\P_{\yp \sim \Rp(w,x)}[\A(x,y,\yp,w)] & \geq \P_{\yp \sim R(w,x)}[\A(x,y,\yp,w)] - TV(R(w,x) \| \Rp(w,x)) \\
 & \geq 1 - \beta_1 - \beta_2
\end{align*}
\noindent which proves our claim.
\qedhere
\end{proof}

\subsection{Proof of Claim \ref{clm:anystruct}}

\begin{proof}
Since ${d(y,\yp) = \iverson{y \neq \yp}}$ and since ${R(x)}$ is a uniform proposal distribution with support on ${\Y(x)}$, we have:
\begin{align} \refstepcounter{equation}
\P_{\yp \sim R(x)}[d(y,\yp) = 1] & = \frac{1}{|\Y(x)|} \sum_{\yh \in \Y(x)}{\iverson{d(y,\yh) = 1}} \nonumber \\
 & = 1 - \frac{1}{|\Y(x)|} \nonumber \\
 & \geq 1 - 1/2 \tag{\theequation.a}\label{eq:anystruct1}
\end{align}
\noindent where the step in eq.\eqref{eq:anystruct1} follows since ${|\Y(x)| \geq 2}$.
\qedhere
\end{proof}

\subsection{Proof of Claim \ref{clm:trees}}

\begin{proof}
Let ${s = (s_1, s_2, s_3 \dots s_v)}$ be the pre-order traversal of $y$.
Let ${\sp = (s_2, s_1, s_3 \dots s_v)}$ be a node ordering where we switched ${s_1}$ with ${s_2}$.
Let ${\Yp(x)}$ be the set of directed spanning trees of $v$ nodes with node ordering $\sp$.\footnote{
We use the node ordering $\sp$ in order to have trees in ${\Yp(x)}$ with all edges different from $y$.
If we use the node ordering $s$ instead, every tree in ${\Yp(x)}$ will contain the edge ${(s_2,s_1)}$, thus no tree in ${\Yp(x)}$ will have all edges different from $y$.}
Let ${\Rp(x)}$ be the uniform proposal distribution with support on ${\Yp(x)}$.
Since ${\Yp(x)}$ is the set of directed spanning trees of $v$ nodes with a specific node ordering, then ${|\Yp(x)| = \prod_{i=2}^v{(i-1)} = (v-1)!}$.
Moreover, since ${d(y,\yp) = \frac{1}{2 (v-1)} \sum_{ij}{|A(y)_{ij} - A(\yp)_{ij}|}}$ and since ${\Rp(x)}$ is a uniform proposal distribution with support on ${\Yp(x)}$, we have:
\begin{align} \refstepcounter{equation}
\P_{\yp \sim R(x)}[d(y,\yp) = 1] & \geq \P_{\yp \sim \Rp(x)}[d(y,\yp) = 1] \nonumber \\
 & = \P_{\yp \sim \Rp(x)}\left[\sum_{ij}{|A(y)_{ij} - A(\yp)_{ij}|} = 2 (v-1)\right] \nonumber \\
 & = \frac{1}{(v-1)!} \sum_{\yh \in \Yp(x)}{\iverson{\sum_{ij}{|A(y)_{ij} - A(\yh)_{ij}|} = 2 (v-1)}} \nonumber \\
 & = \frac{1}{(v-1)!} {\rm\ } \prod_{i=3}^v{(i-2)} \tag{\theequation.a}\label{eq:trees1} \\
 & = 1 - \frac{v-2}{v-1} \nonumber
\end{align}
\noindent where the step in eq.\eqref{eq:trees1} follows from the fact that when choosing the parent for the node in position $i$ in the ordering $\sp$, we have one option less (i.e., the option that is in $y$).
\qedhere
\end{proof}

\subsection{Proof of Claim \ref{clm:dags}}

\begin{proof}
Let ${s = (s_1, s_2, s_3 \dots s_v)}$ be the pre-order traversal of $y$.
Let ${\sp = (s_2, s_1, s_3 \dots s_v)}$ be a node ordering where we switched ${s_1}$ with ${s_2}$.
Let ${\Yp(x)}$ be the set of directed acyclic graphs of $v$ nodes and $b$ parents per node, and with node ordering $\sp$.\footnote{
We use the node ordering $\sp$ in order to have graphs in ${\Yp(x)}$ with all edges different from $y$.
If we use the node ordering $s$ instead, every graph in ${\Yp(x)}$ will contain the edge ${(s_2,s_1)}$, thus no graph in ${\Yp(x)}$ will have all edges different from $y$.}
Let ${\Rp(x)}$ be the uniform proposal distribution with support on ${\Yp(x)}$.
Since ${\Yp(x)}$ is the set of directed acyclic graphs of $v$ nodes and $b$ parents per node, and with a specific node ordering, then ${|\Yp(x)| = \prod_{i=2}^{b+1}{(i-1)} \prod_{i=b+2}^v{\binom{i-1}{b}} = b! \prod_{i=b+2}^v{\binom{i-1}{b}}}$.
Moreover, since ${d(y,\yp) = \frac{1}{b(2v-b-1)} \sum_{ij}{|A(y)_{ij} - A(\yp)_{ij}|}}$ and since ${\Rp(x)}$ is a uniform proposal distribution with support on ${\Yp(x)}$, we have:
\begin{align} \refstepcounter{equation}
\P_{\yp \sim R(x)}[d(y,\yp) = 1] & \geq \P_{\yp \sim \Rp(x)}[d(y,\yp) = 1] \nonumber \\
 & = \P_{\yp \sim \Rp(x)}\left[\sum_{ij}{|A(y)_{ij} - A(\yp)_{ij}|} = b(2v-b-1)\right] \nonumber \\
 & = \left(b! \prod_{i=b+2}^v{\binom{i-1}{b}}\right)^{-1} \sum_{\yh \in \Yp(x)}{\iverson{\sum_{ij}{|A(y)_{ij} - A(\yh)_{ij}|} = b(2v-b-1)}} \nonumber \\
 & = \left(b! \prod_{i=b+2}^v{\binom{i-1}{b}}\right)^{-1} \prod_{i=3}^{b+1}{(i-2)} \prod_{i=b+2}^v{\left(\binom{i-1}{b}-1\right)} \tag{\theequation.a}\label{eq:dags1} \\
 & = \frac{1}{b} {\rm\ } \frac{\binom{b+1}{b}-1}{\binom{b+1}{b}} \prod_{i=b+3}^v{\frac{\binom{i-1}{b}-1}{\binom{i-1}{b}}} \nonumber \\
 & \geq \frac{1}{b} {\rm\ } \frac{\binom{b+1}{b}-1}{\binom{b+1}{b}} \prod_{i=b+3}^v{\frac{\binom{i-1}{2}-1}{\binom{i-1}{2}}} \tag{\theequation.b}\label{eq:dags2} \\
 & = \frac{b v}{(b^2+3b+2)(v-2)} \tag{\theequation.c}\label{eq:dags3} \nonumber \\
 & \geq 1 - \frac{b^2+2b+2}{b^2+3b+2} \nonumber
\end{align}
\noindent where the step in eq.\eqref{eq:dags1} follows from the fact that when choosing the $b$ parents for the node in position $i$ in the ordering $\sp$, we have one option less (i.e., the option that is in $y$).
The step in eq.\eqref{eq:dags2} follows from the fact that the function ${\frac{z-1}{z}}$ is nondecreasing as well as ${\binom{a}{2} \leq \binom{a}{b}}$ for ${a \geq b+2}$ and ${b \geq 2}$.
The step in eq.\eqref{eq:dags3} follows from the fact ${v/(v-2) \geq 1}$ for ${v > 2}$.
\qedhere
\end{proof}

\subsection{Proof of Claim \ref{clm:cardsets}}

\begin{proof}
Since ${\Y(x)}$ is the set of sets of $b$ elements chosen from $v$ possible elements, then ${|\Y(x)| = \binom{v}{b}}$.
Moreover, since ${d(y,\yp) = \frac{1}{2 b} (|y-\yp| + |\yp-y|)}$ and since ${R(x)}$ is a uniform proposal distribution with support on ${\Y(x)}$, we have:
\begin{align} \refstepcounter{equation}
\P_{\yp \sim R(x)}[d(y,\yp) = 1] & = \P_{\yp \sim R(x)}[|y-\yp| + |\yp-y| = 2 b] \nonumber \\
 & = 1 - \P_{\yp \sim R(x)}[|y-\yp| + |\yp-y| < 2 b] \nonumber \\
 & = 1 - \binom{v}{b}^{-1} \sum_{\yh \in \Y(x)}{\iverson{|y-\yh| + |\yh-y| < 2 b}} \nonumber \\
 & = 1 - \binom{v}{b}^{-1} {\rm\ \ } \sum_{i = 0}^{b-1}{\binom{v-b}{i}} \tag{\theequation.a}\label{eq:cardsets1} \\
 & \geq 1 - \binom{v}{b}^{-1} {\rm\ \ } \sum_{i = 0}^{b-1}{\frac{(v-b)^i}{i!}} \tag{\theequation.b}\label{eq:cardsets2} \\
 & = 1 - \binom{v}{b}^{-1} \frac{\exp{v-b} \int_{v-b}^{+\infty}{t^{b-1}\exp{-t}dt}}{(b-1)!} \nonumber \\
 & = 1 - \binom{v}{\floor{\alpha v}}^{-1} \frac{\exp{v-\floor{\alpha v}} \int_{v-\floor{\alpha v}}^{+\infty}{t^{\floor{\alpha v}-1}\exp{-t}dt}}{(\floor{\alpha v}-1)!} \tag{\theequation.c}\label{eq:cardsets3} \\
 & \geq 1 - 1/2 \tag{\theequation.d}\label{eq:cardsets4}
\end{align}
\noindent where the step in eq.\eqref{eq:cardsets1} follows from the fact that for a fixed set $y$ of $b$ elements, if the set $\yh$ has ${b-i}$ common elements with $y$, then there are ${\binom{v-b}{i}}$ possible ways of choosing the remaining $i$ non-common elements in $\yp$ from out of ${v-b}$ possible elements.
The step in eq.\eqref{eq:cardsets2} follows from well-known inequalities for the binomial coefficient.
The step in eq.\eqref{eq:cardsets3} follows from making ${b = \floor{\alpha v}}$.
The step in eq.\eqref{eq:cardsets4} follows for any ${\alpha \in [0,1/2]}$.
\qedhere
\end{proof}

\subsection{Proof of Claim \ref{clm:sparsedata}}

\begin{proof}
Let ${\Delta \equiv \phi(x,y)-\phi(x,\yp)}$.
We also introduce a superindex $p$ for the partitions.
That is, for all ${p \in \PS(x)}$, let ${\Delta^p \equiv \phi(x,y)-\phi(x,\yp)}$ for some ${\yp \in \Y_p(x)}$.
By assumption, since ${\yp \in \Y_p(x)}$ then ${|\Delta^p_p| = b}$ and ${(\forall q \neq p) {\rm\ } \Delta^p_q = 0}$.
Note that ${\norm{\Delta^p}_1 = \sum_{q \in \PS(x)}{|\Delta^p_q|} = |\Delta^p_p| = b}$.
Thus ${|\Delta^p_p|/\norm{\Delta^p}_1 = 1}$ and ${(\forall q \neq p) {\rm\ } \Delta^p_q/\norm{\Delta^p}_1 = 0}$.
Therefore:
\begin{align*}
\norm{\E_{\yp \sim R(x)}\left[ \mu(\Delta) \right]}_2 & = \sqrt{\sum_{q \in \PS(x)}{ \E_{\yp \sim R(x)}\left[\frac{\Delta_q}{\norm{\Delta}_1}\right]^2 }} \\
 & \leq \sqrt{\sum_{q \in \PS(x)}{ \E_{\yp \sim R(x)}\left[\frac{|\Delta_q|}{\norm{\Delta}_1}\right]^2 }} \\
 & = \sqrt{\sum_{q \in \PS(x)}{ \left( \sum_{p \in \PS(x)}{ \P_{\yp \sim R(x)}[\yp \in \Y_p(x)] \; \frac{|\Delta^p_q|}{\norm{\Delta^p}_1} } \right)^2 }} \\
 & = \sqrt{\sum_{q \in \PS(x)}{ \left( \P_{\yp \sim R(x)}[\yp \in \Y_q(x)] \; \frac{|\Delta^q_q|}{\norm{\Delta^q}_1} \right)^2 }} \\
 & = \sqrt{|\PS(x)| \left(\frac{1}{|\PS(x)|}\right)^2} \\
 & = 1/\sqrt{|\PS(x)|}
\end{align*}
\noindent where we used the fact that for a uniform proposal distribution ${R(x)}$, we have ${\P_{\yp \sim R(w,x)}[\yp \in \Y_q(x)] = 1/|\PS(x)|}$.
Finally, since we assume that ${n \leq |\PS(x)|/4}$, we have ${1/\sqrt{|\PS(x)|} \leq 1/(2 \sqrt{n})}$ and we prove our claim.
\qedhere
\end{proof}

\subsection{Proof of Claim \ref{clm:densedata}}

\begin{proof}
Let ${\Delta \equiv \phi(x,y)-\phi(x,\yp)}$.
By assumption ${|\Delta_p| = b}$ for all ${p \in \PS(x)}$.
Note that ${\norm{\Delta}_1 = \sum_{p \in \PS(x)}{|\Delta_p|} = |\PS(x)| \; b}$.
Thus ${|\Delta_p|/\norm{\Delta}_1 = 1/|\PS(x)|}$ for all ${p \in \PS(x)}$.
Therefore:
\begin{align*}
\norm{\E_{\yp \sim R(w,x)}\left[ \mu(\Delta) \right]}_2 & = \sqrt{\sum_{p \in \PS(x)}{ \E_{\yp \sim R(w,x)}\left[\frac{\Delta_p}{\norm{\Delta}_1}\right]^2 }} \\
 & \leq \sqrt{\sum_{p \in \PS(x)}{ \E_{\yp \sim R(w,x)}\left[\frac{|\Delta_p|}{\norm{\Delta}_1}\right]^2 }} \\
 & = \sqrt{|\PS(x)| \left(\frac{1}{|\PS(x)|}\right)^2} \\
 & = 1/\sqrt{|\PS(x)|}
\end{align*}
Finally, since we assume that ${n \leq |\PS(x)|/4}$, we have ${1/\sqrt{|\PS(x)|} \leq 1/(2 \sqrt{n})}$ and we prove our claim.
\qedhere
\end{proof}

\subsection{Proof of Claim \ref{clm:zhang}}

\begin{proof}
Algorithm \ref{alg:zhang} depends solely on the linear ordering induced by the parameter $w$ and the mapping ${\phi(x,\cdot)}$.
That is, at any point in time, Algorithm \ref{alg:zhang} executes comparisons of the form ${\dotprod{\phi(x,y)}{w} > \dotprod{\phi(x,\yh)}{w}}$ for any two structured outputs $y$ and $\yh$.
\qedhere
\end{proof}

\subsection{Proof of Claim \ref{clm:greedylocal}}

\begin{proof}
Algorithm \ref{alg:greedylocal} depends solely on the linear ordering induced by the parameter $w$ and the mapping ${\phi(x,\cdot)}$.
That is, at any point in time, Algorithm \ref{alg:greedylocal} executes comparisons of the form ${\dotprod{\phi(x,y)}{w} > \dotprod{\phi(x,\yh)}{w}}$ for any two structured outputs $y$ and $\yh$.
\qedhere
\end{proof}

\end{document}